\documentclass{article}

\usepackage{multicol}
\usepackage{float}

\usepackage[final, nonatbib]{neurips_2023}

\usepackage{AKI}

\title{Little Exploration is All You Need}

\author{%
  Henry H.H. Chen \\
  ISTBI, Fudan University\\
  \texttt{hhchen21@m.fudan.edu.cn}\\
  \And
  Jiaming Lu \\
  ISTBI, Fudan University\\
  \texttt{jmlu21@m.fudan.edu.cn}\\
}

\begin{document}

\maketitle

\begin{abstract}
The prevailing principle of "Optimism in the Face of Uncertainty" advocates for the incorporation of an exploration bonus, generally assumed to be proportional to the inverse square root of the visit count ($1/\sqrt{n}$), where $n$ is the number of visits to a particular state-action pair. This approach, however, exclusively focuses on "uncertainty," neglecting the inherent "difficulty" of different options. To address this gap, we introduce a novel modification of standard UCB algorithm in the multi-armed bandit problem, proposing an adjusted bonus term of $1/n^\tau$, where $\tau > 1/2$, that accounts for task difficulty. Our proposed algorithm, denoted as UCB$^\tau$, is substantiated through comprehensive regret and risk analyses, confirming its theoretical robustness. Comparative evaluations with standard UCB and Thompson Sampling algorithms on synthetic datasets demonstrate that UCB$^\tau$ not only outperforms in efficacy but also exhibits lower risk across various environmental conditions and hyperparameter settings.
\end{abstract}

\begin{multicols}{2}

\section{Introduction}\label{sec:intro}

The Upper Confidence Bound (UCB) algorithm belongs to a class of index policies designed for tackling the multi-armed bandit problem. It dynamically selects an arm based on an index calculated as  \textit{empirical mean} + \textit{exploration bonus}. This bonus serves to prioritize arms with higher levels of uncertainty in their posterior distributions, thereby mitigating the risk of consistently choosing sub-optimal arms due to initial underestimations. This concept of leveraging optimism to guide decision-making is not limited to multi-armed bandits; it is a foundational principle known as Optimism in the Face of Uncertainty. This principle has found applications across a diverse range of fields, including contextual bandits \citep{abbasi2011improved}, online convex optimization \citep{rakhlin2013optimization}, reinforcement learning \citep{azar2017minimax}, and game theory \citep{daskalakis2021near}.




Over the years, it has been generally assumed that the exploration bonus should be directly proportional to the posterior uncertainty of an arm's reward, based on the data collected.  However, this additional bonus tends to encourage the selection of all arms, even those that require minimal exploration. This can lead to imprecise execution of the exploration-exploitation strategy within a finite time horizon $T$, despite aligning with asymptotic behavior. To illustrate, consider how the standard UCB bonus decays in proportion to $\O(1/\sqrt{N_a(t)})$ as the number of pulls $N_a(t)$ increases. Even when $N_a(t)$ is large, the residual bonus can still impede the exploration of other arms.

We hypothesize that a bonus that converges more quickly to zero as $N_a(T) \rightarrow \infty$. could yield better performance across a variety of problems. This is a noteworthy observation for both theorists and practitioners, and it's a conjecture we address in this paper. Recent work \citep{bayati2020unreasonable} highlights the surprising effectiveness of a greedy algorithm (with zero extra bonus) when the number of arms $K$ is large, specifically when $K=\O(\sqrt{T})$. This suggests that the standard UCB may be overly explorative, at least in settings with many arms. Our simulations indicate that more greedy algorithms can outperform even the Lai \& Robbins lower bound in such settings. 

One possible explanation for this is that as more samples are collected, not only does the uncertainty around individual arms decrease, but the relationships between arms—such as gap estimators—also become clearer. If these gaps are significant, it raises the question of whether it might be beneficial to halt exploration early, even when some uncertainty remains. This lends credence to the effectiveness of the greedy algorithm highlighted by Bayati, especially considering that the number of inter-arm relationships $K(K-1)/2$ far exceeds $K$.



\textbf{Outline of Contributions.\quad} 

In order to reduce the amount of exploration significantly for arms with sufficiently collected data, we introduce a novel algorithm, UCB$^\tau$, which generalizes the standard UCB algorithm by incorporating a parameter $\tau \ge 1/2$. Specifically, we define the index policy as follows:

\begin{equation} \label{eq:UCB-tau}
    I_a(t) = \underbrace{\hat{\mu}_{a}(\mathcal{N}_{a}(t-1))}_{\textbf{Empirical Mean}} + \underbrace{\left( \alpha_a \cdot \frac {\log t}{N_{a}(t-1)} \right)^\tau}_{\textbf{Exploration Bonus}}
\end{equation}
Here, $\alpha_a > \beta_a(\tau)$, and the latter is given by
\begin{equation} \label{eq:beta-a}
    \beta_a(\tau) \triangleq 2 \left(\frac 1{2\tau}\right)^{\frac 1{\tau}} \frac {\sigma_{\ast}^2}{\Delta_a^{2 - \frac 1\tau}}.
\end{equation}
We term $\alpha_a$ as \textit{Exploration mass} and $\tau$ as the \textit{exploration exponent}. As elaborated in Sec. \ref{sec:setup}, UCB$^\tau$ takes into account both the uncertainty and the task difficulty inherent to the bandit problem. Notably, the exploration bonus in UCB$^\tau$ decays at an order of $\O(1 / N^\tau)$, which is considerably faster than that of the standard UCB. This accelerated decay effectively mitigates the bias introduced in arm comparisons when arms have been sufficiently sampled.

The algorithm under consideration is elegantly simple, yet robust enough to allow for a comprehensive regret and risk analysis across a broad spectrum of environment classes. Specifically, we focus on arm rewards that follow a subGaussian distribution, which encompasses most of the reward assumptions found in existing literature. While bounded rewards (hence subGaussian) are the most commonly assumed, we note an exception in the work of \citep{cappe2013kullback}, who studied rewards from the exponential family.

For environment class, we introduce the notation $\Phi_K(\sigma)$ to represent the class of $K$-armed bandits with $\sigma$-subGaussian arms. It's worth noting that the gaps in $\Phi_K(\sigma)$ can be arbitrarily small, irrespective of the subGaussian constant. We also propose a novel set of environment classes, denoted as $\Psi^s_K(\gamma)$, where the weighted noise-gap ratio is regulated. This ensures that both the gaps and the subGaussian constants must go to zero together. Different algorithms are naturally suited to different environment classes. Interestingly, for each
$\tau \ge 1/2$, there exists a corresponding $s$ for which UCB$^\tau$ is minimax efficient within the $\Psi^s(\gamma)$ class. See details in Sec. \ref{subsec:environment class}. In this context, the environment class serves as a form of prior knowledge. Specifically, $\Psi^s(\gamma)$ is the class that needs to be considered for tuning the algorithm to achieve $\O(\log T)$ distribution-dependent regret.



We undertake a comprehensive analysis of various forms of regret, including distribution-dependent regret, minimax regret, and discounted regret when the discount rate $\lambda$ approaches zero. For clarity, we use the term $T$-regret to refer to finite-horizon, undiscounted regret, and $\lambda$-regret to denote discounted, infinite-horizon regret. Theorem \ref{them:little} establishes an 
$\O(\log T)$ $T$-regret, which closely approximates the asymptotic optimality defined by the Lai \& Robbins lower bound. Meanwhile, Theorem \ref{thm:minimax} provides a $\tilde{\O}(T^{1 - \tau} K^\tau)$ upper bound for minimax regret, where $\tilde{O}$ conceals logarithmic terms. This is applicable under the environment class defined by $s = 1 - \frac 1{2\tau}$.

Turning to risk analysis, Theorem \ref{thm:highprob} offers a high-probability bound on pseudo-regret, revealing that the likelihood of pseudo-regret exceeding a given threshold decays polynomially as the threshold increases. Additionally, Theorem \ref{thm:under-explore} examines the worst-case scenarios when the chosen exploration mass is insufficient, particularly when the tunability condition is not met.

In Section \ref{sec:exp}, we provide simulation results that strongly validate our theoretical findings. The results indicate that UCB$^\tau$ is not only significantly more effective but also more robust compared to the standard UCB algorithm. Furthermore, when $\tau > 1/2$ is optimally tuned, the regret closely approximates the Lai \& Robbins lower bound across a broad spectrum of 
$\alpha_a$ values. This stands in contrast to the performance of the standard UCB. Upon testing various choices for $\alpha_a$, we find that $\beta_a(\tau)$, as defined in Eq. \eqref{eq:beta-a}, serves as an exact hyperparameter match for optimal performance. Intriguingly, we also discover that UCB$^\tau$ surpasses the Lai \& Robbins lower bound when dealing with high intrinsic uncertainty and a relatively small $T$.

\textbf{Limitation.\quad} A notable constraint of the UCB$^\tau$ algorithm lies in the tuning of 
$\alpha_a$, which necessitates prior knowledge of $\beta_a(\tau)$. This becomes particularly challenging when the "difficulty" of the environment is unknown beforehand. For instance, in scenarios where variances are fixed and gaps can be arbitrarily small — corresponding to the environment class $\Phi(\sigma)$ as discussed in Section \ref{subsec:environment class} — the choice of $\alpha_a$ could potentially fall below $\beta_a(\tau)$ if a choice of $\tau > 1/2$ is intended to be used. This would result in polynomial regret, as outlined in Theorem \ref{thm:under-explore}.

However, a pragmatic workaround is to conservatively select a sufficiently large 
$\alpha_a$ value that is unlikely to be lower than $\beta_a(\tau)$. Our simulations indicate that such an increase in $\alpha_a$ does not significantly impact performance, especially when 
$\tau > 1/2$.


\textbf{Takeaway for practitioners.\quad} 
In practice, choosing $\tau > 1/2$ yields a significant performance boost. However, an overly aggressive setting of $\tau$ could risk selecting an $\alpha_a$ value that falls below the required threshold. For $1/2\le \tau < 1$, a minimax guarantee exists, as stated in Theorem \ref{thm:minimax}. As $\tau$ increases in this range, the dependence of minimax regret on 
$T$ diminishes. Therefore, it is advisable to select $\tau$ within the range from 1 to 2. A comprehensive summary of these recommendations is provided in Table \ref{tab:summary}.

\begin{algorithm}[H]
    \caption{(general) upper confidence bound policy}
    \label{alg:index}
    \KwIn {bonus $c_{t, n, a}$\\}
    Set $N_a(t) = 0, \hat{\mu}_a(t) = 0$. \\
    \For {$t = 1, \dots, T$} {
        Compute $I_a(t) = \hat{\mu}_a(t) + c_{t, N_a(t-1), a}$.\\
        Pull arm     $A_t=\arg\max_{k\in[K]} I_a(t)$ (break tie randomly). \\
        Receive Reward $X_t = X_{t, A_t}$.\\
        Update $N_a(t), \hat{\mu}_a(t)$.
    }
\end{algorithm}


\section{Formalism}
\label{sec:setup}

\textbf{Notation.\quad} For positive integer $n$, define $[n] \triangleq \{1, \ldots, n\}$. $\# E$ and $|E|$ both represents the numbers of elements in the set $E$. For a sequence of quantities $x_1, \ldots, x_n$, define $x_{1:n} \triangleq (x_1, \ldots, x_n)$. 

The underlying philosophy of letter-picking is that the meaning is not obscured too much when  subscripts and brackets are dropped, so we might drop them at ease whenever we would like to, if no confusion is incurred. Specificly, $T$ for \textit{horizon}, $K$ for number of arms, $N$ for number of pulls, $\Delta$ for gap, $A$ for arm, $X$ for reward, $R$ for regret. Note that in some literature they use "$T_a(n)$" to denote the number of pulls, instead of $N_a(t)$ in this paper. A complete treatment of problem setting can be found in Appendix \ref{sec:problem-setting}, where any not yet defined notations can be found.

\subsection{Warm-up: Stochastic MAB}

Consider the game of stochastic multi-armed bandit. In the game, the agent adaptive chooses among $K$ different arms for $T$ rounds in total, receiving and only observing the reward incurred by their actions. In game theory, one might call it a one-player partial information game with stochastic payoff function. Let $X_{t, a}, t \in[T], a \in [K]$ denote the potential reward of pulling arm $a$ at step $t$. All potential rewards are mutually independent. For fixed arm $a$, all $X_{t, a}$ are all drawn from the same distribution $P_a$. Denote $\mu_a \triangleq \E[X_{t, a}]$ the expected reward. At each step $t$, the agent chooses an arm $A_t\in[K]$, and the stochastic reward $X_t \triangleq X_{t, A_t}$ is revealed. The collection of potential rewards $\phi=(X_{t, a})_{t\in[T], a\in[K]}$ is called an \textit{instance} or \textit{environment}. Let $\mathcal{F}_t = \sigma(X_{1:(t-1)}, A_{1:(t-1)})$ denote the $\sigma$-algebra generated by the trajectory up to step $t$. A policy is simply a conditional distribution $A_t \mid \mathcal{F}_t$, denoted $\pi$. Note that the policy is allowed to be randomized.

The arm with the largest expected reward is called optimal arm, and others are called sub-optimal arms. For convenience of analysis, we assume the uniqueness of optimal arm. Define
\begin{equation}
    A_\ast = \argmax_{a\in[K]} \mu_a, \quad \mu_\ast = \max_{a 
 \in[K]} \mu_a.
\end{equation}
The gap of arm $a$ is defined as
\begin{equation}
    \Delta_a = \mu_\ast - \mu_a,\quad a\in[K].
\end{equation}

Since the semimal work of \citep{lai1985asymptotically}, the notion of \textit{regret} has become the standard measure of performance of bandits in the past decades. The regret is the difference between the sum of expected rewards from pulling the optimal arm and sum of expected rewards generated by the agent. Using regret as measurement has such trait: since all past rewards are weighted equally, the agent needs not only to explore the arms with high uncertainty, but also to exploit the arms with high expected rewards, where these two goals sometimes advocate choosing different arms. This phenomenon is often called \textit{exploration-exploitation dilemma}. Let
\begin{equation} \label{eq: regret-def}
    R_T(\phi, \pi) \triangleq \E\left[ \sum\nolimits_{t=1}^T (\mu_\ast - \mu_{A_t}) \right] 
\end{equation}

$R_T(\phi, \pi)$ is called \textit{distribution-dependent} regret. Eq. \eqref{eq: regret-def} without taking the expectation is sometimes called the pseudo regret, denoted $\hat{R}_T$. Fix an environment class $\Phi$ (see \ref{subsec:environment class} for definition), the minimax regret is defined as the worst-case regret in the entire environment class
\begin{equation}
    R_T(\Phi, \pi) = \max_{\phi\in\Phi} R_T(\phi, \pi).
\end{equation}

Historically, the discounted regret was also considered for measuring bandit algorithms. Let $\lambda > 0$, define the discounted regret 
\begin{equation} \label{eq:lambda-regret}
    R^\lambda(\phi, \pi) = \sum_{t=1}^\infty \E[ e^{- \lambda t} (\mu_\ast - \mu_{A_t})].
\end{equation}
It is well-defined because the the number of arm is finite, hence gaps are bounded. To be brief, $R_T$ will be called the $T$-regret, and $R^\lambda$ will be called the $\lambda$-regret.  The optimal solution of $\lambda$-regret was given by the so-called dynamic allocation index \citep{gittins_dynamic_1979, gittins1979bandit}. But the discount rate was considered fixed, and was quite different from the current study of $T$-regret, where the horizon $T$ is taken for any positive integer, with particular interest over when $T$ is large. To our knowledge, no work in the bandit field has considered the minimization of $\lambda$-regret for all $\lambda > 0$, with particular interest in the case $\lambda$ near $0$. 

The regret only measures the performance of the agent, but not the risk. We use Regret at Risk (RaR) as a measure of risk of multi-armed bandit. 
Let $\alpha \in (0, 1)$, the $\alpha$-regret at risk is defined as
\begin{equation}
    \RaR_T(\alpha) \triangleq \inf \{x>0 : \P(\hat{R}_T \le x) \ge \alpha \}.
\end{equation}


\subsection{Environment Class and Prior Knowledge}\label{subsec:environment class}

So far the setting is too generic for any regret bound to be proven. Certain restrictions must be put on the environment class so that an effective algorithm to tackle this environment class can exist. In this paper, we consider arm rewards that are subGaussian, which encompasses a wide range of distributions. In particular, any Gaussian, bounded, and moreover, any distributions such that the tail of the p.d.f. is lighter than some Gaussian, is subGaussian \citep{wainwright2019high}. Counter-examples are those that have thicker tails, e.g. exponential family, Laplace distribution.
\begin{definition} \label{def:Phi}
    Let $\Phi_K(\sigma)$ denote the environment class of stochastic $K$-armed bandits, where each reward distribution is $\sigma_a$-subGaussian, such that 
    $$
    \max_{a\in[K]}\sigma_a \le \sigma.
    $$
    Moreover, let
    \begin{equation}
        \Phi(\sigma) \triangleq \bigcup_{K=2}^\infty \Phi_K(\sigma), \quad \Phi \triangleq \bigcup_{\sigma \ge 0} \Phi(\sigma). 
    \end{equation}
\end{definition}

 The parameter $\sigma$ upper bounds of the intrinsic uncertainty of reward distribution. It is proportional to the number of samples needed to shrink the confidence interval to a certain width. To our knowledge, previous literature implicitly adopted $\Phi$ or its subset as the environment class. In this paper, we also propose a new type of environment classes, which incorporates the "difficulty" of the regret minimization task. 
\begin{definition} \label{def:Psi}
    For $0\le s \le 1$, let $\Psi_K^s(\gamma)$ denote the environment class of stochastic $K$-armed bandits, where each reward distribution is $\sigma_a$-subGaussian and each gap is $\Delta_a$, such that 
    $$ 
    \frac {\max_{a\in[K]}\sigma_a}{ \left(\min_{a\in[K] : \Delta_a > 0 } \Delta_a\right)^s} \le \gamma.
    $$
    Moreover, let 
    \begin{equation}
        \Psi(\gamma) \triangleq \bigcup_{K=2}^\infty \Psi_K(\gamma),\quad \Psi \triangleq \bigcup_{\gamma \ge 0} \Psi(\gamma).
    \end{equation}
\end{definition} 
When $s=0$, it reduces to $\Psi^s_K(\sigma)\mid_{s=0} = \Phi_K(\sigma)$. When $s=1$, denoting $\Psi_K(\gamma) \triangleq \Psi^{s}_K(\gamma)\mid_{s=1}$, the constant $\gamma$ upper bounds the noise-gap ratio $\gamma_a \triangleq \frac {\sigma_a}{\Delta_a}$. An important property of $\Psi_K(\gamma)$ is that it is homogeneous up to shifting and re-scaling of arm rewards, i.e. if $\phi = (X_{t, a}) \in \Psi_K(\gamma)$, then so is $(\alpha \cdot X_{t, a} + \beta)$, for any $\alpha > 0$ and $\beta \in \R$. Statistically, $\gamma$ is proportional to the number of samples needed to distinguish two different arms. If $\gamma_a$ is large, it is hard to distinguish arm $a$ from the optimal arm in moderate size of pulls. 

Notice that different choices of $s$ in Definition \ref{def:Psi} represent quite distinct type of environment classes. In particular, no one includes another. Therefore, it is important to design different algorithms to tackle different environment classes. The environment class represents \textit{prior knowledge}, which is a set of bandit instances $\K$, where we know for sure that the instance confronting is drawn from the set $\K$. Gaining more prior knowledge leads to the shrinkage of the set $\K$. We will use the term "environment class" and "prior knowledge" interchangably. 



\subsection{Upper Confidence Bound}

Let $c_{t, s, a}$ denote a generic bonus for $t \in[T], n\in[t], a \in [K]$. We require $c_{t, n, a} > 0$ and adapted to $\mathcal{F}_t$. Define the index function
\begin{equation}
    I_a(t) = \hat{\mu}_{a}(\mathcal{N}_{a}(t-1)) + c_{t, N_a(t-1), a}.
\end{equation}
And 
\begin{equation}
    I_{t, n, a} = \hat{\mu}_a(\mathcal{H}_a(n)) + c_{t, n, a}.
\end{equation}

where $\mathcal{N}_a(t)$ is the set of pulling times in the first $t$ rounds, and $\mathcal{H}_a(n)$ is the set of first $n$ pulls. Thus, at each step, the agent selects the arm with the index of maximal value (see Algorithm \ref{alg:index}). Index policies is horizon-free, computation-light and easily-implemented, compared to strategies like Thompson Sampling \citep{russo2014learning}, which involves sampling from the posterior distribution. 


\section{Regret Analysis} \label{sec:regret}

In this section we prove regret bounds for UCB$^\tau$. The $T$-regret is reformulated as
\begin{equation}
    R_T = \sum_{a\in[K]} \Delta_a \E[N_{a}(T)].
\end{equation}
The $\lambda$-regret is reformulated as
\begin{equation}
    R^\lambda = \sum_{n=1}^\infty \E[\exp(-\lambda H_a(n))].
\end{equation}
We can express $R^\lambda$ in terms of $R_T$
\begin{equation}
    R^\lambda = (1 - e^{-\lambda}) \sum_{n=1}^{\infty} e^{-n\lambda} R_n.
\end{equation}

\subsection{Distribution-dependent regret}
The standard path to bound the regret \eqref{eq: regret-def} is to obtain a bound of $\E[N_{a}(T)]$, of which the regret $R_T$ is a linear combination. The original method \citep{auer2002finite} was to define a "Good" event $G$ such that $N_a(T)$ is small conditioned on $G$, and $G^c$ happens with low probability. Then $\E[N_a(T)] \le \E[N_a(T) \mid G] + T \cdot \P(G^c)$, where both term is small with $G$ chosen properly. \citep{audibert2009exploration} provided a more refined approach, splitting directly into two parts the event $\{A_t = a\}$, of which $N_a(T)$ is a linear combination. We follow the latter approach. Nevertheless, our analysis is more delicate, since they use a refined bonus to make the second part convergent, while we merely use the standard bonus to achieve that, and our analysis works for any $\tau \ge 1/2$.

\begin{theorem}\label{them:little}
For UCB$^\tau$ algorithm with $\tau \ge 1/2$, and for sub-optimal $a$ with $\alpha_{\ast} > \beta_a(\tau)$, for any $\eta > 0$ with $2 \tau \eta < 1$, we have
\begin{equation}\label{eq:NTA}
    \E[N_{a}(T)] \le \frac {\alpha_a}{2[(1/2\tau - \eta)\Delta_a]^{1/\tau}}  \log T + \O(1),
\end{equation}
where $\beta_a(\tau)$ is defined in Eq. \eqref{eq:beta-a} and the $\O(1)$ term may depend on $(\phi, \pi, \eta)$.

\begin{remark}
    From Theorem \ref{them:little}, we immediately get $R_T = \O(\log T)$ for UCB$^\tau$ tuned with $\alpha_a > \beta_a(\tau)$.
\end{remark}

\begin{remark}
    By multiplying $\beta_a(\tau)$ on both numerator and denominator, Eq. \eqref{eq:NTA} is rewritten as
    \begin{equation}\label{eq:NTA-re}
        \E[N_{a}(T)] \le (1 - 2 \tau\eta)^{-1/\tau} \frac {\alpha_a}{\beta_a(\tau)} \frac {2 \sigma_{\ast}^2}{\Delta_a^2}  \log T + \O(1).
    \end{equation}
    We may choose $\alpha_a$ close to but larger than $\beta_a(\tau)$ and let $\eta \rightarrow 0$. Then
    \begin{equation}
        \limsup_{T\rightarrow \infty} \frac {\E[N_a(T)]}{\log T} \le 2\sigma_\ast^2 / \Delta_a^2 + \O(\alpha_a - \beta_a(\tau)).
    \end{equation}
    Notice that the constant $2$ in the right hand side is the best achievable result, guaranteed by the lower bound provided by \citep{lai1985asymptotically}. Thus, we have shown that when $\alpha_a$ is tuned appropriately (close to but larger than $\beta_a(\tau)$), the distribution-dependent $T$-regret is almost asymptotically optimal. In particular for $\tau = 1/2$, our bound on standard UCB is sharper than all previous results, e.g. \citep{auer2002finite,Auer:10,audibert2009exploration}.
\end{remark}


\end{theorem}

\subsection{Minimax regret} \label{subsec: minimax regret}

\begin{theorem} \label{thm:minimax}
    Let $s = 1 - \frac {1}{2\tau}$, $\gamma > 0$. Choose $\alpha_a = 2 \gamma^2$. Then for all $\phi \in \Psi^s(\gamma)$ with $1/2 \le \tau < 1$, 
    \begin{equation}
        R_T \le (1 + 2\gamma^2)T^{1 - \tau} (K \log T)^{\tau} + \O(1).
    \end{equation}

\end{theorem}

\begin{proof}
    \begin{equation}
        \begin{aligned}
            R_T = & \sum_{a : \Delta_a > 0} \Delta_a \E[N_a(T)] \\
            = & T \Delta + \sum_{\Delta_a > \Delta} \Delta_a \E[N_a(T)] \\
            \le & T \Delta +  2K \Delta^{1 - 1/\tau}\gamma^2  \log T + \O(1) \\
        \end{aligned}    
    \end{equation}

    Now choose $\Delta = T^{-\tau} (K \log T)^{\tau}$ gives the result.
\end{proof}

\textbf{How to tune $\tau$ and $\alpha_a$ with given prior knowledge. \quad} An \textit{algorithm} is a mapping from prior knowledge to policy, i.e. $\K \mapsto \pi$. This definition is perhaps narrower than one usually encounters, but in this paper, we will use the term "algorithm" precisely according to this definition. Hence, UCB$^\tau$ becomes an algorithm only if we specify the rule of tuning the parameters $\alpha_a$ according to given prior knowledge. Examples are given in the following. First we present the definition of \textit{tunability}.

\begin{definition}
    An algorithm for multi-armed bandit is called \textit{tunable} for prior knowledge $\K$, if the regret of its image policy $\pi$ satisfies $R_T(\phi, \pi) = \O(\log T)$ for any $\phi \in \K$. 
\end{definition}

Astute readers might find that $\beta_a(\tau)$ in Eq. \eqref{eq:beta-a} depends on the gaps, making UCB$^\tau$ not tunable for prior knowledges like $\Phi(\sigma)$ provided $\tau > \frac 12$. In fact, different choices of $\tau$ handles different prior knowledges. Suppose the prior knowledge is $\Phi(\sigma) \cap \Psi(\gamma)$. Simple algebra gives
\begin{equation}
    \beta_a(\tau) 
    < 2 (\sigma^2 + \gamma^2).
\end{equation}
We see that 
\begin{itemize}[leftmargin=*]
    \item $\tau = 1/2$ is tunable for $\Phi(\sigma)$ with tuning rule $\alpha_a = (2 + \delta)\sigma^2$.
    \item $1/2 < \tau < \infty$ is tunable for $\Phi(\sigma) \cap \Psi(\gamma)$ with tuning rule $\alpha = 2(\sigma^2 + \gamma^2)$.
    \item $\tau = \infty$ is tunable for $\Psi(\gamma)$ with tuning rule $\alpha_a = (2 + \delta) \gamma^2$.
\end{itemize}

\subsection{Discounted regret near 0} \label{subsec: discounted}

Theory along only gets you so far. Although researchers have been pursuing the finite-time regret bounds, so that these bounds not only hold asymptotically, but for any fixed finite horizon, the results are still far off the situation. It turns out that the regret does not grow linearly with the logarithmic time, but rather has two stages. When the horizon is small, the regret could be even lower than Lai \& Robbins lower bound (interpreted as finite-time bound), while the regret eventually grows linearly with the logarithmic time when the horizon is large. 

Let $\lambda > 0$, define the discounted regret 
\begin{equation}
    R^\lambda = \sum_{t=1}^\infty \E[ e^{- \lambda t} (\mu_\ast - \mu_{A_t})].
\end{equation}
 We call $R^\lambda$ the $\lambda$-regret and $R_T$ the $T$-regret. Persuiting the asymptotics of $T$-regret might sacrifice finite-time performance. Therefore, we might turn to minimize the growth rate of $\lambda$-regret as $\lambda \rightarrow 0$. First we define the notion of consistent policy for $\lambda$-regret.

\begin{definition} \citep{lai1985asymptotically}
    A policy is $T$-consistent, if
    \begin{equation}
        R_T = o(T^a) ,\quad \forall a > 0.
    \end{equation}  
\end{definition}

\begin{definition}
    A policy is $\lambda$-consistent, if
    \begin{equation}
        R^\lambda = o(1 / \lambda^a) ,\quad \forall a > 0.
    \end{equation}    
\end{definition}
The next proposition shows the equivalence of $\lambda$-consistency and $T$-consistency.
\begin{proposition}
    A policy is $\lambda$-consistent if and only if it is $T$-consistent.
\end{proposition}

\begin{proof}
    $T$-consistency $\Rightarrow$ $\lambda$-consistency: Fix $0 < a < 1$. Suppose $R_T \le C T^a$ for some constant $C>0$, over all positive integer $T$. Notice that 
    \begin{equation}
        R^\lambda = (1 - e^{-\lambda}) \sum_{n=1}^\infty e^{-n\lambda} R_n.
    \end{equation}
    Hence by assumption
    \begin{equation}
        \begin{aligned}
            R^\lambda & \le C(1 - e^{-\lambda}) \sum_{n=1}^\infty e^{-n\lambda} n^a \\
                      & \le C (1 - e^{-\lambda})^{-a} \\
                      & \le C e^{a\lambda} / \lambda^a.
        \end{aligned}
    \end{equation}
    where the second inqeuality follows from the Hölder inequality \ref{lemma:holder}.
    
    $\lambda$-consistency $\Rightarrow$ $T$-consistency: Fix $0 < a < 1$, suppose $R^\lambda \le C \lambda^{-a}$ for some constant $C>0$, over all $\lambda > 0$. Then
    \begin{equation} \label{eq:RTRl}
        R_T \le e^{T\lambda} R^\lambda \le C e^{T\lambda} \lambda^{-a}.
    \end{equation}
    Now take $\lambda = 1 / T$ gives the result.
\end{proof}

\begin{theorem} \label{thm:lai-robbins} \citep{lai1985asymptotically}
    \begin{equation}
        \liminf_{T\rightarrow \infty} \frac {R_T} {\log(T)} \ge \sum_{a : \Delta_a > 0} \frac {1}{\KL(P_a \| P_\ast)}
    \end{equation}
\end{theorem}

\begin{theorem}
    \begin{equation}
        \liminf_{\lambda \rightarrow 0^+} \frac {R^\lambda} {\log(1/\lambda)} \ge \sum_{a : \Delta_a > 0} \frac {1}{\KL(P_a \| P_\ast)}
    \end{equation}
\end{theorem}

\begin{proof}
    The first inequality of Eq. \eqref{eq:RTRl} gives us
    \begin{equation}
        R^\lambda \ge e^{-\epsilon} R_{\lfloor \epsilon/\lambda \rfloor}, \quad \epsilon > 0.
    \end{equation}
    Divide $\log (1/\lambda)$, and take liminf on both sides,
    \begin{equation}
        \begin{aligned}
            & \liminf_{\lambda \rightarrow 0^+} \frac {R^\lambda}{\log(1/\lambda)} \\
            \ge &\liminf_{\lambda \rightarrow 0^+} \frac {R_{\lfloor \epsilon/\lambda \rfloor}}{\log(1/\lambda)} \\
            = & \liminf_{T \rightarrow \infty} \frac {R_{T}}{\log T -\log\epsilon} \\
            = & \liminf_{T \rightarrow \infty} \frac {R_{T}}{\log T }.
        \end{aligned}
    \end{equation}
    The rest follows by Theorem \ref{thm:lai-robbins}.
\end{proof}

Albeit the equivalence of consistency, algorithms that minimize $T$-regret and $\lambda$-regret have completely different characteristics. A $T$-regret minimization algorithms must not select sub-optimal arms too often, regardless of when and how they are selected. On the contrary, a $\lambda$-regret minimization algorithms can select sub-optimal arms as many times as they want -- as long as not too early. As a consequence, the $\lambda$-regret, focusing on not selecting sub-optimal arms too early, might be a better measure of performance if the interest is not too far-sighted that any cost in the early state can be completely omitted. It is only when $T\rightarrow \infty$ and $\lambda \rightarrow 0$ that these two notions eventually be one.

\textbf{Open problem.} In the field of Reinforcemnet Learning \citep{konda_actor-critic_1999,mnih_playing_2013,lillicrap_continuous_2015}, discounted return is used as a measure of performance, where the discounted ratio is considered fixed. \citep{jin2018q} studied Q-learning in the regret minimization regime. We have shown that UCB$^\tau$ minimizes the $\lambda$-regret for $\lambda$ near $0$. It is natural to ask the following question: 
\begin{quote}
    \textit{Is is possible to design RL algorithms so as to minimize the discounted regret near $0$?}    
\end{quote}

\section{Risk Analysis} \label{sec: risk}

\subsection{High probability bound}

\begin{theorem} \label{thm:highprob}
    For $u \ge \frac {\alpha_a \log T}{(\Delta_a - \varepsilon - \varepsilon_2)^{-1/\tau} }$,
    \begin{equation}
        \P(N_a(T) \ge u) \le T\exp(-\frac {u \varepsilon_2^2}{2\sigma_a^2}) + \frac {2\sigma_\ast^2}{\delta^2} u^{\frac {\beta(\varepsilon - \delta, k) -\alpha_\ast}{\beta(\varepsilon-\delta, k)}}.
    \end{equation}

\end{theorem}

\subsection{The price of under-exploration} 
Theorem \ref{them:little} has shown that, given exploration exponent $\tau\ge 1/2$, the UCB$^\tau$ algorithm has provable regret bound $O(\log T)$ if the exploration rate $\alpha$ surpasses certain threshold. Here we analyze what happens when the exploration rate falls below the threshold. Under exploration aggravates the bifurcation phenomenon, which can be seen from Lemma \ref{lemma:nta2} as the series non-convergent. So instead of getting a zeta function, we get a non-convergent partial sum with certain growth rate. Theorem \ref{thm:under-explore} provide an upper bound that is, unfortunately, $\widetilde{O}(T^{1 - \frac { \alpha_{\ast} }{\beta_{a}}})$. This shows that the influence of under-exploration is not severe provided that the amount of under-exploration is small and $T$ is only of moderate size.
\begin{theorem} \label{thm:under-explore}
    When $\alpha_{\ast} \le \beta_{a}(\tau)$ for some sub-optimal $a$, 
\begin{equation} \label{eq:r1-fail}
    \E[N_{a}^{(2)}(T, \varepsilon)] \le \frac {2\sigma_\ast^2}{\delta^2} \left( 1 +  T^{1 - \frac { \alpha_{\ast} }{\beta_{a}(\tau)}} \log T \right)
\end{equation}
\end{theorem}

\begin{proof}
    Reminiscent of Lemma \ref{lemma:nta2} and then use Lemma \ref{lemma:zeta2}.    
\end{proof}

\section{Discussion}\label{sec:discussion}

\begin{table}[H]
\small
\setlength{\tabcolsep}{4pt}
\begin{tabular}{l|cccc}
    & Performance & Risk & Tunable & Minimax Regret\\
        \hline
    $\tau=\frac 12$             & \ding{56} & \ding{56} & $\Phi$  & $\tilde{\O}(\sqrt{KT})$ \\
    $\frac 12 < \tau < 1$ &  \ding{52} & \ding{52} & $\Psi^{1 - \frac {1}{2\tau}}$ & $\tilde{\O}(T^{1 - \tau} K^{\tau})$ \\
    $ 1 \le \tau < \infty$ &  \ding{52} & \ding{52} & $\Psi^{1 - \frac {1}{2\tau}}$ &  \\
    $\tau = \infty$            &  \ding{52} & \ding{52} & $\Psi$  \\
\end{tabular}
\caption{Summary of the merits of UCB$^\tau$}
\label{tab:summary}
\end{table}




\textbf{ETC = UCB$^\infty$.\quad} We have seen when $\tau > 1/2$, tuning $\alpha_a$ requires prior knowledge on $\Psi(\gamma)$, which represents difficulty of the task. When $\tau = \infty$, it depends on $\Psi(\gamma)$ along. In fact, UCB$^\infty$ is an interesting case as we now elaborate. As $\tau$ increases, the bonus decays to zero for $N_a(t)$ large enough. One might suspect that UCB$^\tau$ reduces to the Greedy algorithm (zero extra bonus). However, it is a little more delicate. 

Let's look at the bonus of UCB$^\tau$ for $\alpha_a = \frac {(2+\delta)\sigma_\ast^2}{\Delta_a^{2}} = \beta_a(\infty)$.
\begin{equation}
    I_a(t) = 
    \left\{ 
    \begin{aligned}
        \hat{\mu}_a(N_a(t-1)),\quad & \text{if }N_a(t-1) \ge \frac {(2+\delta)\sigma_\ast^2 \log t}{\Delta_a^2}, \\
        \infty ,\quad & \text{otherwise.}
    \end{aligned} 
    \right. 
\end{equation}
Thus, UCB$^\tau$ switches between forced sampling and greedy policy. At each step, each arm should collect at least $\O(\log t)$ samples, otherwise will be sampled by force. Ties in $\infty$ can be broken either at random or in order. 

Explore-Then-Commit \citep{garivier2016explore} is another algorithm that combines forced sampling and greedy. It first samples each arm $m$ times, then chooses the arm with the largest empirical mean all the way to the end (without further updating). They also gave a fully sequential algorithm for two-arm problem, which was essentially UCB$^\infty$ here, only we consider $K$ arms rather than just two.

The tuning of ETC \citep{lattimore2020bandit} is given by 
\begin{equation}
    m = \max \{1, \left\lceil  \frac {4\sigma^2}{\Delta^2}\log \left( \frac {\sigma^2 T}{4\Delta^2} \right) \right\rceil \}
\end{equation}
We see that the sampling ratio of UCB$^\infty$ matches the $m$ in ETC in terms of $\sigma, \Delta$ and $\log t$. Hence we can view UCB$^\infty$ as an \textit{anytime} (independent of horizon $T$) version of ETC. 

\textbf{Bifurcation of Greedy Algorithm.\quad } It is well-known that greedy algorithm fails in many scenarios. However, in some cases greedy performs arbitrarily well. 

\begin{example} \label{ex: greedy-good}
    When $\sigma_{\ast} = 0$, the greedy algorithm has regret
    \begin{equation}
        R_T \le \sum_{a\in[K]:\Delta_a > 0} (\Delta_a + \frac {2\sigma_a^2}{\Delta_a})
    \end{equation}
\end{example} 

\begin{example} \label{ex: greedy-bad}
    $K=2, A_\ast =1, \sigma_1> 0, \sigma_2 = 0$. Rewards are Gaussian. The greedy algorithm has regret
    \begin{equation}
        R_T \ge \P(X_{1, 1} < \mu_2) (T-1)
    \end{equation}
\end{example}
This bifurcation proves that the bonus cannot be reduced to zero.

\section{Experiment}\label{sec:exp}

\begin{figure*}[ht]
  \centering
  \includegraphics[width=1\textwidth]{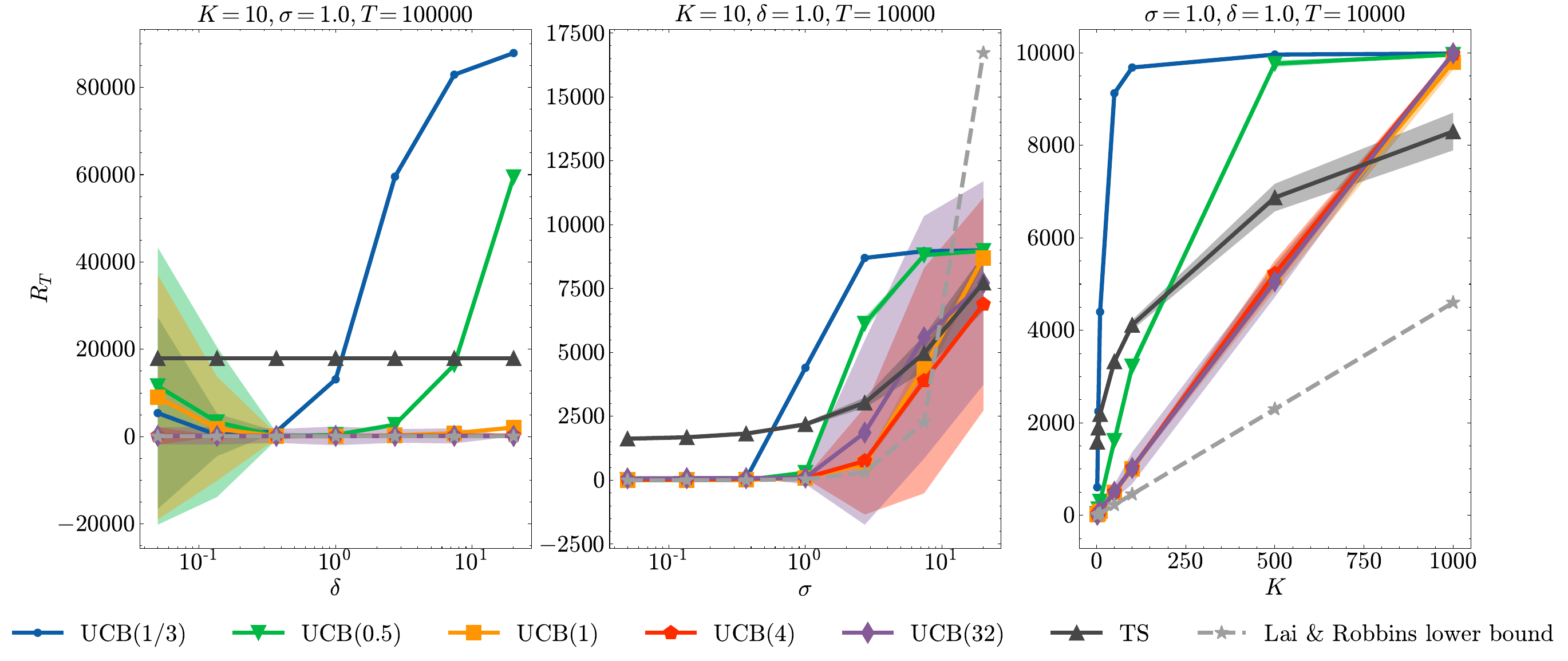}
  \caption{Regret in Various Configurations of the Algorithm in Gaussian Reward Scenarios, with a Focus on the Regret Mean and Standard Deviation (Shaded Region) at T=10,000 over 4,096 Repetitions. The left graph represents the scenario when $\delta$ changes while other factors remain constant. The middle graph illustrates the case when $\sigma$ varies while other factors remain unchanged. The rightmost graph depicts the scenario when the number of arms changes while other factors remain constant.}
  \label{fig:gaussian_threeinone_T=10000}
\end{figure*}

To further substantiate our theoretical findings, we conducted a comprehensive numerical simulation with enriched configurations on a stochastic multi-armed bandit problem, featuring Gaussian and Bernoulli rewards (Due to space constraints, we have included additional experimental details in the appendix \ref{sec: Rest of Experiments}.). We juxtaposed the UCB$^\tau$ algorithm with $\tau = \frac 13, \frac 12, 1,2,4,32$ against the widely employed Thompson sampling algorithm and the $\epsilon-$greedy algorithm. On the whole, the results reveal a notable advantage for the UCB$^\tau$ algorithm when the parameter $\tau > \frac 12$.

Our experiments were executed within a hyperparameter grid. Herein, we enumerate a subset of the settings:
\begin{enumerate}
    \item Number of arms. $K=2,5,10,50,100, 500, 1000$.
    \item Noise-gap ratio $\gamma \triangleq \frac {\sigma}{\Delta}$. $\sigma \in (e^{-3}, e^3)$ and $\Delta=1$.
    \item Exploration mass. $\delta \in (e^{-3}, e^3)$.
    \item Number of Repetitions: 4096.
    \item Duration of Algorithm: [1, 100,000].
\end{enumerate}

First, we test the tightness of the constant $\beta_a(\tau)$ by setting $\alpha_a = \delta \cdot \beta_a(\tau)$ where $\delta \in [e^{-2}, e^2]$ . When $\delta>1$, the regret is low and the variance of regret is also low. When $\delta < 1$, the regret and variance become high. See Fig. \ref{fig:gaussian_threeinone_T=10000} (left) for results. This shows that $\beta_a(\tau)$ is the exact best-match of hyperparameter.
Regardless of how $\delta$ varies, when $\tau>\frac{1}{2}$, the algorithm UCB$^\tau$ consistently outperforms the original UCB and TS algorithms, approaching the Lai \& Robbins lower bound.

Second, we compare the performance under different noise-gap ratios, results in Fig. \ref{fig:gaussian_threeinone_T=10000} (middle). As $\gamma$ increases, all algorithms suffer not only larger regret but also larger variance. Eventually, every algorithm fails. Even when $\gamma$ is high, UCB$^\tau$ still outperforms standard UCB, and comparable to TS. On the contrary, when $\gamma$ is small, TS performs poorly, while both standard UCB and UCB$^\tau$ has good performance.

Third, we compare the performance under different arms, results in Fig. \ref{fig:gaussian_threeinone_T=10000} (right). When $K$ is large, UCB fails tragically. One possible explanation is that the cost of injected bias is emplified when $K$ is large. From $K=2$ to $K=1000$, the result shows that UCB$^\tau$ out performs standard UCB and TS. As $K$ grows, the advantage is larger.

 \newpage
\bibliographystyle{agsm}
\bibliography{references}

@article{lai1985asymptotically,
  title={Asymptotically efficient adaptive allocation rules},
  author={Lai, Tze Leung and Robbins, Herbert},
  journal={Advances in applied mathematics},
  volume={6},
  number={1},
  pages={4--22},
  year={1985}
}

@article{auer2002finite,
  title={Finite-time analysis of the multiarmed bandit problem},
  author={Auer, Peter and Cesa-Bianchi, Nicolo and Fischer, Paul},
  journal={Machine learning},
  volume={47},
  number={2},
  pages={235--256},
  year={2002},
  publisher={Springer}
}

@book{lattimore2020bandit,
  title={Bandit algorithms},
  author={Lattimore, Tor and Szepesv{\'a}ri, Csaba},
  year={2020},
  publisher={Cambridge University Press}
}

@article{bayati2020unreasonable,
  title={Unreasonable effectiveness of greedy algorithms in multi-armed bandit with many arms},
  author={Bayati, Mohsen and Hamidi, Nima and Johari, Ramesh and Khosravi, Khashayar},
  journal={Advances in Neural Information Processing Systems},
  volume={33},
  pages={1713--1723},
  year={2020}
}

@article{russo2014learning,
  title={Learning to optimize via posterior sampling},
  author={Russo, Daniel and Van Roy, Benjamin},
  journal={Mathematics of Operations Research},
  volume={39},
  number={4},
  pages={1221--1243},
  year={2014},
  publisher={INFORMS}
}

@article{audibert2009exploration,
  title={Exploration--exploitation tradeoff using variance estimates in multi-armed bandits},
  author={Audibert, Jean-Yves and Munos, R{\'e}mi and Szepesv{\'a}ri, Csaba},
  journal={Theoretical Computer Science},
  volume={410},
  number={19},
  pages={1876--1902},
  year={2009},
  publisher={Elsevier}
}

@article{Auer:10,
  title={UCB revisited: Improved regret bounds for the stochastic multi-armed bandit problem},
  author={Auer, Peter and Ortner, Ronald},
  journal={Periodica Mathematica Hungarica},
  volume={61},
  pages={55--65},
  year={2010}
}

@article{lattimore2018refining,
  title={Refining the confidence level for optimistic bandit strategies},
  author={Lattimore, Tor},
  journal={The Journal of Machine Learning Research},
  volume={19},
  number={1},
  pages={765--796},
  year={2018},
  publisher={JMLR. org}
}

@inproceedings{garivier2011kl,
  title={The KL-UCB algorithm for bounded stochastic bandits and beyond},
  author={Garivier, Aur{\'e}lien and Capp{\'e}, Olivier},
  booktitle={Proceedings of the 24th annual conference on learning theory},
  pages={359--376},
  year={2011},
  organization={JMLR Workshop and Conference Proceedings}
}

@article{garivier2022kl,
  title={KL-UCB-switch: optimal regret bounds for stochastic bandits from both a distribution-dependent and a distribution-free viewpoints},
  author={Garivier, Aur{\'e}lien and Hadiji, H{\'e}di and Menard, Pierre and Stoltz, Gilles},
  journal={The Journal of Machine Learning Research},
  volume={23},
  number={1},
  pages={8049--8114},
  year={2022},
  publisher={JMLRORG}
}

@article{cappe2013kullback,
  title={Kullback-Leibler upper confidence bounds for optimal sequential allocation},
  author={Capp{\'e}, Olivier and Garivier, Aur{\'e}lien and Maillard, Odalric-Ambrym and Munos, R{\'e}mi and Stoltz, Gilles},
  journal={The Annals of Statistics},
  pages={1516--1541},
  year={2013},
  publisher={JSTOR}
}

@inproceedings{audibert2009minimax,
  title={Minimax Policies for Adversarial and Stochastic Bandits.},
  author={Audibert, Jean-Yves and Bubeck, S{\'e}bastien and others},
  booktitle={COLT},
  volume={7},
  pages={1--122},
  year={2009}
}

@article{gittins1979bandit,
  title={Bandit processes and dynamic allocation indices},
  author={Gittins, John C},
  journal={Journal of the Royal Statistical Society Series B: Statistical Methodology},
  volume={41},
  number={2},
  pages={148--164},
  year={1979},
  publisher={Oxford University Press}
}

@book{wainwright2019high,
  title={High-dimensional statistics: A non-asymptotic viewpoint},
  author={Wainwright, Martin J},
  volume={48},
  year={2019},
  publisher={Cambridge university press}
}

@article{abbasi2011improved,
  title={Improved algorithms for linear stochastic bandits},
  author={Abbasi-Yadkori, Yasin and P{\'a}l, D{\'a}vid and Szepesv{\'a}ri, Csaba},
  journal={Advances in neural information processing systems},
  volume={24},
  year={2011}
}

@inproceedings{azar2017minimax,
  title={Minimax regret bounds for reinforcement learning},
  author={Azar, Mohammad Gheshlaghi and Osband, Ian and Munos, R{\'e}mi},
  booktitle={International Conference on Machine Learning},
  pages={263--272},
  year={2017},
  organization={PMLR}
}

@article{jin2018q,
  title={Is Q-learning provably efficient?},
  author={Jin, Chi and Allen-Zhu, Zeyuan and Bubeck, Sebastien and Jordan, Michael I},
  journal={Advances in neural information processing systems},
  volume={31},
  year={2018}
}

@article{garivier2016explore,
  title={On explore-then-commit strategies},
  author={Garivier, Aur{\'e}lien and Lattimore, Tor and Kaufmann, Emilie},
  journal={Advances in Neural Information Processing Systems},
  volume={29},
  year={2016}
}

@article{rakhlin2013optimization,
  title={Optimization, learning, and games with predictable sequences},
  author={Rakhlin, Sasha and Sridharan, Karthik},
  journal={Advances in Neural Information Processing Systems},
  volume={26},
  year={2013}
}

@article{daskalakis2021near,
  title={Near-optimal no-regret learning in general games},
  author={Daskalakis, Constantinos and Fishelson, Maxwell and Golowich, Noah},
  journal={Advances in Neural Information Processing Systems},
  volume={34},
  pages={27604--27616},
  year={2021}
}

@article{zimmert2021tsallis,
  title={Tsallis-inf: An optimal algorithm for stochastic and adversarial bandits},
  author={Zimmert, Julian and Seldin, Yevgeny},
  journal={The Journal of Machine Learning Research},
  volume={22},
  number={1},
  pages={1310--1358},
  year={2021},
  publisher={JMLRORG}
}

@article{gittins_dynamic_1979,
	title = {A dynamic allocation index for the discounted multiarmed bandit problem},
	volume = {66},
	issn = {0006-3444, 1464-3510},
        doi = {10.1093/biomet/66.3.561},
	language = {en},
	number = {3},
	urldate = {2023-10-15},
	journal = {Biometrika},
	author = {Gittins, J. C. and Jones, D. M.},
	year = {1979},
	pages = {561--565},
}

@misc{mnih_playing_2013,
	title = {Playing {Atari} with {Deep} {Reinforcement} {Learning}},
	doi = {10.48550/arXiv.1312.5602},
	urldate = {2023-10-18},
	publisher = {arXiv},
	author = {Mnih, Volodymyr and Kavukcuoglu, Koray and Silver, David and Graves, Alex and Antonoglou, Ioannis and Wierstra, Daan and Riedmiller, Martin},
	month = dec,
	year = {2013},
	note = {arXiv:1312.5602 [cs]},
}

@inproceedings{konda_actor-critic_1999,
	title = {Actor-{Critic} {Algorithms}},
	volume = {12},
	urldate = {2023-10-18},
	booktitle = {Advances in {Neural} {Information} {Processing} {Systems}},
	publisher = {MIT Press},
	author = {Konda, Vijay and Tsitsiklis, John},
	year = {1999},
}

@misc{lillicrap_continuous_2015,
	title = {Continuous control with deep reinforcement learning},
	language = {en},
	urldate = {2023-10-18},
	journal = {arXiv.org},
	author = {Lillicrap, Timothy P. and Hunt, Jonathan J. and Pritzel, Alexander and Heess, Nicolas and Erez, Tom and Tassa, Yuval and Silver, David and Wierstra, Daan},
	month = sep,
	year = {2015},
}

\end{multicols}

\newpage

\appendix
\newtheorem{lemma2}{Lemma}[section]
\newtheorem{fact2}[lemma2]{Fact}


\section{Problem Settings and Notations} \label{sec:problem-setting}

Let 
\begin{equation}
    \mathcal{N}_{a}(t) \triangleq \{ n \in[t] : A_n = a\}, \quad N_{a}(t) \triangleq |\mathcal{N}_{a}(t)|
\end{equation}
denote the set and the number of steps where arm $a$ is pulled in the first $t$ rounds, respectively.

Let 
\begin{equation}
    H_a(n) \triangleq \min\{ t \in [T] : A_t = a, t > H_a(n-1) \}, \quad \mathcal{H}_a(n) \triangleq \{H_a(1), \ldots, H_a(n)\}
\end{equation}
denote the $n$-th hitting time and the first $n$ hits, respectively. If the $N_a(T) < n$, define $H_a(n) = \infty$. Define $H_a(0) = 0$.

For any subset $
\S \subset [T]$, define
\begin{equation}
    \hat{\mu}_{a}(\S) \triangleq \frac {1}{|\S|} \sum_{t\in\S} X_{t, a}
\end{equation}
as the average reward of arm $a$ over the steps in $\S$. 

To avoid double subscripts (as much as possible),  define
\begin{equation}
\begin{aligned}
 &\alpha_\ast \triangleq \alpha_{A_\ast}, \quad \sigma_\ast \triangleq \sigma_{A_\ast}, \quad N_{\ast}(\cdot) \triangleq N_{A_\ast}(\cdot), \quad I_\ast(\cdot) \triangleq I_{A_\ast}(\cdot) \\
 &\hat{\mu}_{\ast}(\cdot) \triangleq \hat{\mu}_{A_\ast}(\cdot), \quad \T_{\ast}(\cdot) \triangleq T_{A_\ast}(\cdot)  \quad c_{t, s, \ast} \triangleq c_{t, s, A_\ast}
\end{aligned}
\end{equation}
which stand for short-hand notations related to the optimal arm.

\section{Related Works}
\label{sec:related}

The multi-armed bandit was first solved in discounted setting, where the optimal solution is given by the so-called dynamic allocation index \citep{gittins1979bandit,gittins_dynamic_1979}. It turns out that the undiscounted case is much more subtle than it seems, since the excess loss might be unbounded under improper balance of exploration and exploitation. In the seminal paper, \citep{lai1985asymptotically} introduced Upper Confidence Bound (UCB) algorithm that achieves $\O(\log T)$ regret for tackling the stochastic multi-armed bandit problem. They also established a lower bound of regret $\Omega(\log T)$, that is, under mild assumptions, any bandit algorithm suffers
\begin{equation} \label{eq:lai-robbins}
    \liminf_{T\rightarrow \infty} \frac {R_T}{\log T} \ge \sum_{a\in[K] : \Delta_a > 0} \frac {\Delta_a}{\KL(P_\ast \| P_a)}.
\end{equation}
Since asymptotic regret bound is sometimes far from predictive of the practical performance, finite-time regret analysis is considered as more robust guarantees and becomes the standard requirement. \citep{auer2002finite,Auer:10} was the first to prove $\mathcal{O}(\log T)$ finite-time regret bound for the standard UCB algorithm. An algorithm that achieves equality in Eq. \eqref{eq:lai-robbins} is considered \textit{asymptotically optimal}. Improvements over standard UCB are made to achieve asymptotic optimality in the setting of bounded rewards \citep{garivier2011kl}, one-parameter exponential family \citep{cappe2013kullback}, uni-variance Gaussian \citep{lattimore2018refining}, among other assumptions of reward distribution. Refinement can also be made by computing more statistics, therefore reducing the prior knowledge injected at the beginning time. UCB-V \citep{audibert2009exploration} uses variance estimate as uncertainty quantifier to control the bonus size, resulting a more refined regret bound when variances are heterogeneous. Apart from the results of distribution-dependent regret bounds, it has been shown that the distribution-free regret has lower bound $\Omega(\sqrt{KT})$. The first algorithm that matches this bound is given by \citep{audibert2009minimax}. Recently works attempt to design algorithms that establish simultaneous optimality in multiple worlds. \citep{garivier2022kl} established simultaneous optimality in both distribution-dependent and distribution-free settings. \citep{zimmert2021tsallis} proposed an algorithm based on Online Mirror Descent (OMD) that is optimal in both stachastic and adversarial multi-armed bandit with bounded rewards.

In the field of Reinforcement Learning (RL), the optimism in the face of uncertainty (OFU) principle has also been a focal point to tackle exploration. \citep{azar2017minimax} introduced the UCB-VI algorithm, while \cite{jin2018q} offered variants UCB-H and UCB-B, which use upper confidence bounds for exploration as opposed to basic $\varepsilon$-greedy strategies. But no experimental results were provided.

\section{Technical Results}

\begin{fact2} \label{fact: Hoefdding inequality} (Hoeffding inequality for subGaussians) \citep{wainwright2019high}
    We say an $\R$-valued random variable $X$ is $\sigma$-subGaussian, for some $\sigma > 0$, if $\E[e^{\lambda(X - \E[X])}] \le e^{\frac 12 \lambda^2 \sigma^2}$ for all $\lambda\in\R$. If so, for any $\epsilon > 0$, we have
    \begin{equation} \label{eq: Hoeffding inequality}
        \P\left( X - \E[X] > \epsilon \right) \le \exp(-\frac {\epsilon^2}{2\sigma^2})
    \end{equation}
    Eq. \eqref{eq: Hoeffding inequality} is the Hoeffding inequality for subGaussians.

    Suppose $X_1, \ldots, X_n$ is an adapted sequence of $\R$-valued random variables, such that $X_i$ is $\sigma_i$-subGaussian conditioned on $X_1, \ldots, X_{i-1}$, for $i \in [n]$, then
    \begin{equation} \label{eq: Azuma-Hoeffding inequality}
        \P(\frac 1n \sum_{i=1}^n X_i - \frac 1n \sum_{i=1}^n \E[X_i] > \epsilon ) \le \exp(-\frac {n^2\epsilon^2}{2\sum_{i=1}^n \sigma_i^2}).
    \end{equation}
    Eq. \eqref{eq: Azuma-Hoeffding inequality} is the Azuma-Hoeffding inequality for subGaussians.

\end{fact2}

\begin{fact2}\label{lemma:holder} (Hölder inequality) Suppose $p, q > 0$ and $\frac 1p +\frac 1q = 1$. For any sequences $(x_n) \in \ell_+^p, (y_n) \in \ell_+^q$,
\begin{equation}
    \sum_{n=1}^\infty x y \le \left(\sum_{n=1}^\infty x^p \right)^{1/p} \left(\sum_{n=1}^\infty x^q \right)^{1/q}.
\end{equation}
\end{fact2}

\begin{lemma2}\label{lemma: butterfly} (Butterfly inequality)
For all $a,b \ge 0, 0 \le k \le 1$, 
\begin{equation*}
    a + b \ge \frac {a^k b^{1-k}}{k^{k}(1-k)^{1-k}}
\end{equation*}
where we define $0^0 = 1$.
\end{lemma2}

\begin{proof}
    For $k=0,1$ the assertion is obvious. Assume $0<k<1$. Let $f(x) = (\frac ak)^x (\frac {b}{1-k})^{1-x}$, then $f(x)$ is a convex function at $0<x < 1$. Hence $(1 - s) f(0) + s f(1) \ge f(s)$ for all $0 < s < 1$. In particular, $(1 - k) f(0) + k f(1) \ge f(k)$. Now plug in the definition of $f(\cdot)$ yields the assertion.
\end{proof}

\begin{lemma2} \label{lemma:zeta2}
Let $\zeta_N(s) = \sum_{n=1}^N n^{-s}$. Then for $0 \le s  \le 1$, 
\begin{equation*}
    \zeta_N(s) \le 1 + N^{1 - s} \log N
\end{equation*}
\end{lemma2}

\begin{proof}
    First, since $f(x)=x^{-s}$ is non-increasing function, 
    \begin{equation*}
        \zeta_N(s) \le 1 + \int_1^N x^{-s} dx 
    \end{equation*}
    For $s = 0, 1$, the claim is trivial. Suppose $0 < s < 1$, then, using integration by part,
    \begin{equation*}
        \int_1^N x^{-s} dx = \int_1^N x^{1 - s} d\log x = N^{1-s} \log N - \int_1^N \frac {x^{-s}}{1 - s} \log x dx \le N^{1-s} \log N 
    \end{equation*}
    The desired result is obtained.
\end{proof}

\begin{lemma2} \label{techlemma:s-1}
    Let $\zeta(s) = \sum_{n=1}^\infty n^{-s}$, defined for $s > 1$. Then
    \begin{equation}
        \zeta(s) \le \frac {s}{s-1}.
    \end{equation}
\end{lemma2}

\section{Proofs in Section \ref{sec:regret}} \label{sec:proof-thm1}

The decomposition is given by
\begin{equation}\label{eq:sep-of-var}
    \begin{aligned}
            & \{A_t = a\} \\
            \subseteq & \{I_a(t) \ge I_\ast(t)\} \\
            \subseteq & \{ I_a(t) \ge \mu_a + \Delta_a - \varepsilon \} \bigcup \{ I_\ast(t) \le \mu_\ast - \varepsilon \}
    \end{aligned}
\end{equation}
Define, for $\varepsilon > 0$,
\begin{equation}
    N_{a}^{(1)}(t, \varepsilon) \triangleq \sum\nolimits_{t=1}^T \ind\{A_{t} = a, I_a(t) \ge \mu_a + \Delta_a - \varepsilon\} , \quad
N_{a}^{(2)}(t, \varepsilon) \triangleq \sum\nolimits_{t=1}^T\ind\{ I_\ast(t) \le \mu_\ast - \varepsilon \}.
\end{equation}
Thus Eq. \eqref{eq:sep-of-var} gives 
\begin{equation}
    N_{a}(T) \le N_{a}^{(1)}(T, \varepsilon) + N_{a}^{(2)}(T, \varepsilon). 
\end{equation}
The rest of the work is to bound $N_{a}^{(1)}$ and $N_{a}^{(2)}$ respectively. 

\begin{lemma} \label{lemma:nta1}
    For $\varepsilon_1, \varepsilon_2 > 0$ satisfying $0<\varepsilon_1<\Delta_a-\varepsilon_2$,
    \begin{equation}\label{eqn:n1}
        \E[N_{a}^{(1)}(T, \varepsilon_1)] \le 1 + \frac {2\sigma_a^2}{\varepsilon_2^{2}} + \# \{ s \in [T - 1] : c_{T,s,a} \ge  \Delta_a - \varepsilon_1 - \varepsilon_2 \},
    \end{equation}
\end{lemma}

\begin{proof}[\textbf{proof of Lemma \ref{lemma:nta1}}]  
    

    \begin{equation} \label{eq:nta1-long}
        \begin{aligned}
            \E[N_{a}^{(1)}(T, \varepsilon_1)]
            =  &\E \left[ \sum\nolimits_{t=1}^T  \ind\{A_{t} = a, I_a(t) \ge \mu_a + \Delta_a - \varepsilon_1\} \right] \\ 
            \le & \E[ 1 + \sum\nolimits_{n=2}^{N_a(T)} \ind\{ I_{H_a(n),n-1,a} \ge \mu_a + \Delta_a - \varepsilon_1 \}] \\
            \le & \E[ 1 + \sum\nolimits_{n=1}^{T-1} \ind\{ I_{T,n,a} \ge \mu_a + \Delta_a - \varepsilon_1 \}] \\
            =   & 1 + \sum\nolimits_{n=1}^{T-1} \P(I_{T,n,a} \ge \mu_a + \Delta_a - \varepsilon_1 ) \\
            \le & 1 + u + \sum_{n=u+1}^{T-1} \P(I_{T,n,a} \ge \mu_a + \Delta_a - \varepsilon_1 ) \\
        \end{aligned}
    \end{equation}
    where $u$ in $\textbf{(I4)}$ is defined as
    \begin{equation} \label{eq:nta1-u}
        u \triangleq \# \{ n \in [T-1] : c_{T,n,a} \ge \Delta_a - \varepsilon_1 - \varepsilon_2 \}.
    \end{equation}
    The summand in the last line of Eq. \eqref{eq:nta1-long} is
    \begin{equation}
        \P(\hat{\mu}_a(\mathcal{H}_a(n)) + c_{T, n, a} \ge \mu_a + \Delta_a - \varepsilon_1 ).
    \end{equation}
    But since $n > u$, by Eq. \eqref{eq:nta1-u}, the above quantity is no greater than
    \begin{equation}
        \P(\hat{\mu}_a(\mathcal{H}_a(n)) + c_{T, n, a} \ge \varepsilon_2 ).
    \end{equation}
    By Doob's optional sampling theorem, $X_{H_a(i), a} \mid X_{H_a(1), a}, \ldots X_{H_a(i-1)}$ shares the same distribution with $X_{1, a}$. Hence $(X_{H_a(i)} - \mu_a)$ is a martingale. Apply Azuma-Hoeffding inequality (Fact \ref{fact: Hoefdding inequality}), we have
    \begin{equation} \label{eq:nta1-exp}
        \sum_{s=u+1}^{T-1} \P(I_{T,s,a} \ge \mu_a + \Delta_a - \varepsilon_1 ) \le \sum\nolimits_{s=u+1}^{T-1} \exp\left(-\frac {s\varepsilon_2^2}{2\sigma_a^{2}}\right) \le \frac {2\sigma_a^2}{\varepsilon_2^{2}}.
    \end{equation}
    Now combine \eqref{eq:nta1-long} \& \eqref{eq:nta1-exp} gives the result.

\end{proof}


Unlike $N_{a}^{(1)}$ that increases as the optimistic bonus grows larger, the $N_{a}^{(2)}$ decreases as the optimistic bonus increases. 
This is the time where the optimistic bonus becomes a savior, 
when encountering bizarre scenarios where the optimal arm $A_\ast$ is under-sampled and under-estimated at the same time. 
The existence of gap lends natural privilege to select $A_\ast$ over other arms, 
and as long as $A_\ast$ is selected frequently enough, repeatedly selecting it will not be a problem. It only requires a small number of samples to identify this advantage. Lemma \ref{lemma:nta2} shows that the inferior sampling probability in $N_{a}^{(2)}$ is bounded by the optimistic bonus in a smooth way, resulting in their sum converging to a finite value as $T\rightarrow \infty$. This desirable property cannot be attained without the optimistic bonus, since in the greedy case, the worst-case situation occurs at $N_{\ast}(t-1) = 1$, causing $N_{a}^{(2)}$ to grow in an order of $O(T)$.
  
\begin{lemma} \label{lemma:nta2}
    Suppose $0 < k < 1$. Suppose $c_{t, s, a}$ satisfies
    \begin{itemize}[leftmargin=*]
        \item[a)] $s \cdot (c_{t, s, \ast})^{2k}$ is a non-decreasing function of $s$ for any $t \in [T]$.
        \item[b)] $c_{t, 1, \ast}^{2k} > \alpha_\ast \log t$ for any $t \in[T]$.
    \end{itemize}
    Let $B(k) \triangleq k^k (1 - k)^{1-k}$. Define
    \begin{equation} \label{eq:beta-ek}
        \beta(\varepsilon, k) = \frac {2 B(k)^2 \sigma_{\ast}^2}{\varepsilon^{ 2 - 2k}}.    
    \end{equation}
    When $\alpha_{\ast} > \beta(\varepsilon, k)$, we have 
    \begin{equation}
         \E[N_{a}^{(2)}(T, \varepsilon)] \le \frac {2\sigma_\ast^2}{\delta^2} \frac {\alpha_\ast}{\alpha_\ast - \beta(\varepsilon - \delta, k)}.
    \end{equation}
    for any $\delta$ with $0 < \delta < \varepsilon$ and $\beta(\varepsilon, k) < \beta(\varepsilon-\delta, k) < \alpha_\ast$.
\end{lemma}

\begin{remark}
    Lemma \ref{lemma:nta2} shows that, when $\alpha_\ast$ is sufficiently large, $N_a^{(2)}$ is upper bounded by a constant, thus has zero contribution to the asymptotic growth of regret. 
\end{remark}

\begin{proof} [\textbf{proof of Lemma \ref{lemma:nta2}}] 

    By definition 
    \begin{equation}
        \E[N_{a}^{(2)}(T, \varepsilon)] =  \sum\nolimits_{t = 1}^T \P \left( I_{t, N_a(t-1), a} \le \mu_\ast - \varepsilon \right).
    \end{equation}
    Take $N_a(t-1) = 1, \ldots t-1$ and apply union bound, we have
    \begin{equation} \label{eq:nta2-doublesum}
        \E[N_{a}^{(2)}(T, \varepsilon)] \le \sum\nolimits_{t=1}^T \sum\nolimits_{n=1}^{t-1} \P\left(I_{t, n, a} \le \mu_\ast - \varepsilon \right).
    \end{equation}
    Apply Azuma-Hoeffding inequality, the summand is no greater than
    \begin{equation}
        \exp \left( - \frac {n(\varepsilon+ c_{t, n, \ast} + \delta)^2}{2\sigma_\ast^2} \right).
    \end{equation}
    Simple algebra gives
    \begin{equation}
        \begin{aligned}
             &\exp \left( - \frac {n(\varepsilon+ c_{t, n, \ast} + \delta)^2}{2\sigma_\ast^2} \right) \\
             \le& \exp \left( - \frac {n((\varepsilon-\delta) + c_{t, n, \ast} + \delta)^2}{2\sigma_\ast^2} \right) \\
             \le& \exp\left( - \frac { (\varepsilon-\delta)^{2-2k}n(c_{t, n, \ast})^{2k} }{2\, B(k)^2\sigma_{\ast}^2} \right) \cdot \exp\left(- \frac {n\delta^2}{2\sigma_\ast^2}  \right)
        \end{aligned}
    \end{equation}
    The left term, by condition (b), is no greater than
    \begin{equation}
        \exp\left(- \frac {c_{t, 1, \ast}^{2k}}{\beta(\varepsilon-\delta, k)}\right)
    \end{equation}
    which only depends on $t$ but not $n$. Hence the double summation of Eq. \ref{eq:nta2-doublesum} can be decoupled and bounded as
    \begin{equation}
        \begin{aligned}
            & \sum\nolimits_{t=1}^T \exp\left(- \frac {c_{t, 1, \ast}^{2k}}{\beta(\varepsilon-\delta, k)}\right) \sum_{s=1}^{\infty}{\exp\left(-\frac {s\delta^2}{2\sigma_\ast^2} \right)} \\
            \le & \frac {2\sigma_\ast^2}{\delta^2} \zeta\left(\frac {\alpha_{\ast}}{\beta(\varepsilon-\delta, k)} \right) \\
            \le & \frac {2\sigma_\ast^2}{\delta^2} \frac {\alpha_\ast}{\alpha_\ast - \beta(\varepsilon - \delta, k)}
        \end{aligned}
    \end{equation}
    where the last line is due to Technical Lemma \ref{techlemma:s-1}. The destination is reached.
    
\end{proof}

Combine Lemma \ref{lemma:nta1} \& \ref{lemma:nta2} leads to the upper bound of $\E[N_{a}(T)]$ in Theorem \ref{them:little}, which in turn leads to a tight regret bound for UCB$^\tau$ algorithm. 
 
The first term on the right-hand side of Eq. \eqref{eqn:n1} represents the threshold $s$ that must be reached for $c_{t, s, a}$ to fall below the gap $\Delta_a-\epsilon_1-\epsilon_2$. In the proof, we show that once this threshold is reached, the algorithm switches to the exploitation mode in an adaptive way. Additionally, since sufficient samples are collected, the inferior sampling rate decays exponentially.

\begin{proof} [\textbf{proof of Theorem \ref{them:little}}]
Apply Lemma \ref{lemma:nta1} to be bonus $c_{t,s,a} = (\alpha_a \cdot \frac {\log t}{s})^\tau$ gives
\begin{equation} \label{eq:thm1-nta1-all}
    \E[n_{T,a}^{(1)}(\varepsilon_1)] \le  \# \{ s \in [T-1] : \left( \alpha_a \cdot \frac {\log t}{s} \right)^\tau \ge \Delta_a - \varepsilon_1 - \varepsilon_2 \} + 1 + 2\varepsilon_2^{-2} \sigma_a^2.
\end{equation}
The first term is computed as
\begin{equation} \label{eq:thm1-nta1-1}
    \begin{aligned}
        & \#\{ s \in [T-1] : \left( \alpha_a \cdot \frac {\log T}{s} \right)^\tau \ge \Delta_a - \varepsilon_1 - \varepsilon_2 \} \\
        = & \#\{ s \in [T-1]: s \le \alpha_a (\Delta_a - \varepsilon_1 - \varepsilon_2)^{-1/\tau} \log T \} \\
        &\le  \alpha_a (\Delta_a - \varepsilon_1 - \varepsilon_2)^{-1/\tau}\log T.
    \end{aligned}
\end{equation}
Now takes $\varepsilon_1 = \Delta_a(1 - \frac 1{2\tau})$, $\varepsilon_2 = \Delta_a \eta$, and summarize Eq. \eqref{eq:thm1-nta1-all} \eqref{eq:thm1-nta1-1} gives 
\begin{equation} \label{eq:nta1-thm1}
    \begin{aligned}
        \E[n_{T,a}^{(1)}(\varepsilon_1)] & \le \alpha_a \Delta_a^{-1/\tau}(\frac 1{2\tau} - \eta)^{-1/\tau} \log T + 1 + \frac {2\sigma_a^2}{\Delta_a^2 \eta^2} \\
        = & \left(\frac {\alpha_a}{\beta_a(\tau)} \right) (1 - 2\tau \eta)^{-1/\tau}\frac {2\sigma_\ast^2}{\Delta_a^2} \log T + 1 + \frac {2\sigma_a^2}{\Delta_a^2 \eta^2}
    \end{aligned}
\end{equation}
where in the last line we multiplied $\beta_a(\tau)$ on both denominator and numerator to the first term.

The next step is to apply Lemma \ref{lemma:nta2} to the bonus. We take $\varepsilon = \varepsilon_1$ and $k = 1 / (2\tau)$ so that the quantity $\beta(\varepsilon, k)$ in Eq. \eqref{eq:beta-ek} coincides with $\beta(\tau)$ in Eq. \eqref{eq:beta-a}. 

Take $\delta$ such that
\begin{equation}
    \beta_a(\tau) < \beta(\varepsilon_1 - \delta, k) < \alpha_\ast
\end{equation}
Then Lemma \ref{lemma:nta2} reads 
\begin{equation} \label{eq:nta2-thm1}
         \E[n_{T, a}^{(2)}(\varepsilon_1)] \le \frac {2\sigma_\ast^2}{\delta^2} \zeta\left(\frac {\alpha_{\ast}}{\beta(\varepsilon_1-\delta, k)} \right).
\end{equation}
Combine Eq. \eqref{eq:nta1-thm1} \& \eqref{eq:nta2-thm1} gives the result.

\end{proof}

\section{Proofs in Section \ref{sec: risk}}

First we show that 
\begin{equation}
    \begin{aligned}
        & \{N_a(T) > u \} \\
      = & \bigcup_{t=u+1}^T \{ H_a(u+1) = t\} \\
      \subseteq & \bigcup_{t=u+1}^T \{ I_{t,u,a} \ge \mu^\ast \} \cup \{H_a(u+1) = t, I_{t,N_\ast(t-1),\ast} \le \mu_\ast \}
    \end{aligned}
\end{equation} 
So that
\begin{equation} \label{eq:NaT-u}
    \P(N_a(T) > u) \le \sum_{t=u+1}^T \P(I_{t,u,a} \ge \mu_a + \Delta_a - \varepsilon) + \sum_{t=u+1}^T \sum_{n=1}^{t-u} \P(I_{t, n, \ast} \le \mu_\ast - \varepsilon)
\end{equation}
Let $u$ be the smallest integer such that
\begin{equation}
    c_{T, u, a} \le \Delta_a - \varepsilon - \varepsilon_2
\end{equation}
which implies
\begin{equation}
    u \ge \frac {\alpha_a \log T}{(\Delta_a - \varepsilon - \varepsilon_2)^{-1/\tau} }.
\end{equation}
Then the first term of Eq. \eqref{eq:NaT-u} is bounded by
\begin{equation}
    \begin{aligned}
        & \sum_{t=u+1}^T \P(I_{t,u,a} \ge \mu_a + \Delta_a - \varepsilon) \\
        \le & T \exp(-\frac {u \varepsilon_2^2}{2\sigma_a^2}) \\
    \end{aligned}
\end{equation}

The second term is bounded by
\begin{equation}
    \begin{aligned}
        & \sum_{t=u+1}^T \sum_{n=1}^{t-u} \P(I_{t, n, \ast} \le \mu_\ast - \varepsilon) \\
        \le & \sum_{t=u+1}^T \sum_{n=1}^{t-u} \exp\left(-\frac {n(\varepsilon + c_{t,n,\ast})^2}{2\sigma_\ast^2} \right) \\
        \le & \sum_{t=u+1}^T \sum_{n=1}^{t-u} \exp\left( - \frac { (\varepsilon-\delta)^{2-2k}n(c_{t, n, \ast})^{2k} }{2\, B(k)^2\sigma_{\ast}^2} \right) \cdot \exp\left(- \frac {n\delta^2}{2\sigma_\ast^2}  \right) \\
        \le & \sum\nolimits_{t=u+1}^T \exp\left(- \frac {c_{t, 1, \ast}^{2k}}{\beta(\varepsilon-\delta, k)}\right) \sum_{n=1}^{\infty}{\exp\left(-\frac {n\delta^2}{2\sigma_\ast^2} \right)} \\
        \le & \frac {2\sigma_\ast^2}{\delta^2} \sum\nolimits_{t=u+1}^T \exp\left(- \frac {\alpha_\ast \log t}{\beta(\varepsilon-\delta, k)}\right) \\
        \le &  \frac {2\sigma_\ast^2}{\delta^2} u^{\frac {\beta(\varepsilon - \delta, k) -\alpha_\ast}{\beta(\varepsilon-\delta, k)}}
    \end{aligned}
\end{equation}
which decays polynomially with $u$.






\section{Rest of Experiments \ref{sec:exp}}\label{sec: Rest of Experiments}

\begin{figure*}[ht]
  \centering
  \includegraphics[width=1\textwidth]{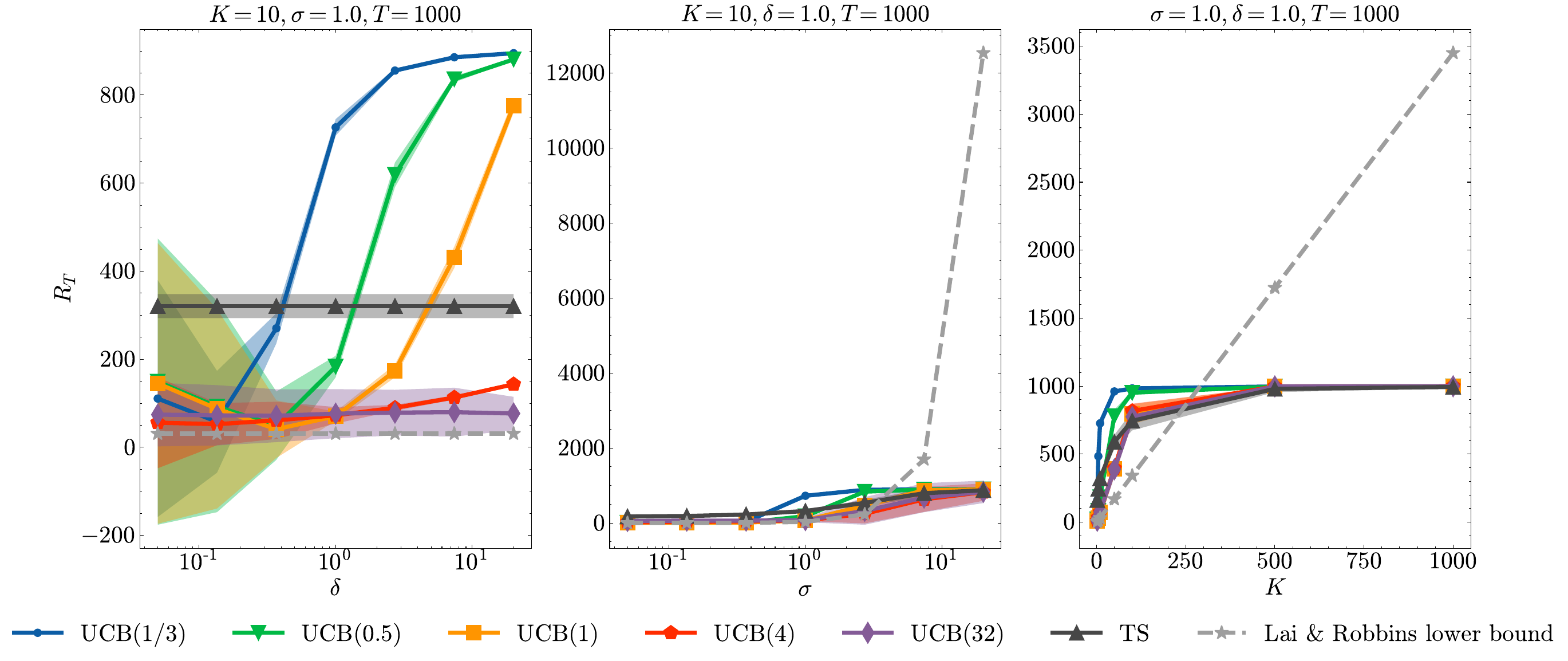}
  \caption{Regret in Various Configurations of the Algorithm in Gaussian Reward Scenarios, with a Focus on the Regret Mean and Standard Deviation (Shaded Region) at T=1000 over 4,096 Repetitions. The left graph represents the scenario when $\delta$ changes while other factors remain constant. The middle graph illustrates the case when $\sigma$ varies while other factors remain unchanged. The rightmost graph depicts the scenario when the number of arms changes while other factors remain constant.}
  \label{fig:gaussian_threeinone_T=1000}
\end{figure*}

\begin{figure*}[ht]
  \centering
  \includegraphics[width=1\textwidth]{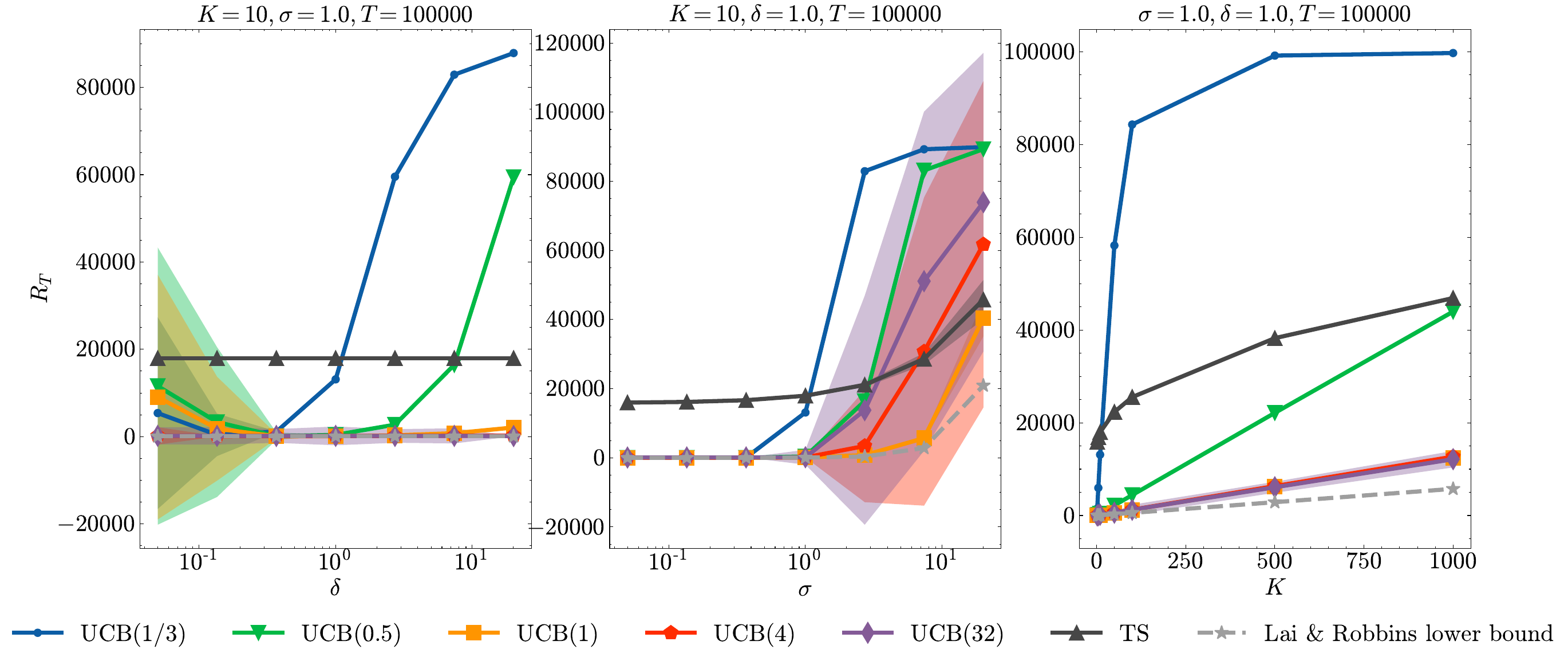}
  \caption{Various $\delta$}
  \caption{Regret in Various Configurations of the Algorithm in Gaussian Reward Scenarios, with a Focus on the Regret Mean and Standard Deviation (Shaded Region) at T=100000 over 4,096 Repetitions. The left graph represents the scenario when $\delta$ changes while other factors remain constant. The middle graph illustrates the case when $\sigma$ varies while other factors remain unchanged. The rightmost graph depicts the scenario when the number of arms changes while other factors remain constant.}
\end{figure*}

\end{document}